\documentclass[a4paper,fleqn]{cas-dc}

\usepackage[numbers]{natbib}
\usepackage{graphicx}
\usepackage{subfigure}
\usepackage{amsmath}
\usepackage{amsthm}
\usepackage{booktabs}
\usepackage{algorithm}
\usepackage{algorithmic}
\usepackage{color}
\usepackage{times}
\usepackage{soul}
\usepackage{url}

\def\tsc#1{\csdef{#1}{\textsc{\lowercase{#1}}\xspace}}
\tsc{WGM}
\tsc{QE}

\begin{document}
\let\WriteBookmarks\relax
\def\floatpagepagefraction{1}
\def\textpagefraction{.001}

\shorttitle{}

\shortauthors{Xu.eal}  

\title [mode = title]{Semi-Supervised Learning with Pseudo-Negative Labels for Image Classification}

\author[1]{Hao Xu}

\author[1]{Hui Xiao}

\author[1]{Huazheng Hao}

\author[1]{Li Dong}

\author[2]{Xiaojie Qiu}

\author[1]{Chengbin Peng}

\cormark[1] 
\ead{pengchengbin@nbu.edu.cn}

 \affiliation[1]{organization={Ningbo University},
             city={Ningbo},
             country={China}}
 \affiliation[2]{organization={Zhejiang Keyongtai Automation Technology Co., Ltd.},
             city={Ningbo},
             country={China}}

\cortext[1]{Corresponding author}

\begin{abstract}
    Semi-supervised learning frameworks usually adopt mutual learning approaches with multiple submodels to learn from different perspectives. 
    To avoid transferring erroneous pseudo labels between these submodels, a high threshold is usually used to filter out a large number of low-confidence predictions for unlabeled data.  
    However, such filtering can not fully exploit unlabeled data with low prediction confidence.
    To overcome this problem, in this work, we propose a mutual learning framework based on pseudo-negative labels. Negative labels are those that a corresponding data item does not belong. In each iteration, one submodel generates pseudo-negative labels for each data item, and the other submodel learns from these labels. The role of the two submodels exchanges after each iteration until convergence. By reducing the prediction probability on pseudo-negative labels, the dual model can improve its prediction ability. 
    We also propose a mechanism to select a few pseudo-negative labels to feed into submodels.
    In the experiments, our framework achieves state-of-the-art results on several main benchmarks. Specifically, with our framework, the error rates of the 13-layer CNN model are 9.35\% and 7.94\% for CIFAR-10 with 1000 and 4000 labels, respectively. 
    In addition, for the non-augmented MNIST with only 20 labels,  the error rate is 0.81\% by our framework, which is much smaller than that of other approaches.
    Our approach also demonstrates a significant performance improvement in domain adaptation.
\end{abstract}

  
\maketitle

\section{Introduction}
    Deep learning is widely used in many areas, and the performance of deep learning models \cite{gao2022decoupled} heavily relies on the amount of training data. 
    However, in many real-world scenarios \cite{khaki2021deepcorn,luo2017adaptive,chen2019semisupervised,yu2018adaptive}, labeled data are often limited, and the annotation for unlabeled data can usually be expensive. 
    In such cases, a semi-supervised learning framework can be adopted.

    Semi-supervised learning frameworks include generative-based models \cite{kingma2014semi}, graph-based models \cite{luo2018smooth}, consistency-based regularization \cite{laine2016temporal,tarvainen2017mean,miyato2018virtual,berthelot2019mixmatch,xie2020unsupervised,ke2019dual,chen2020semi}, self-training with pseudo-labels \cite{feng2020semi,sohn2020fixmatch}, and so on. \
    
    Among them, self-training methods can expand the training set by producing pseudo labels for unlabeled data to improve the model performance. Nevertheless, single models are not robust to noisy data.
    Inspired by DML \cite{zhang2018deep}, a natural idea is to simultaneously train two independently initialized models, and predictions of one submodel can be used as the learning target for the other submodel. \
    
    To avoid transferring erroneous predictions to each other and alleviate parameter coupling between submodels in the early stages of training, a dual student framework \cite{ke2019dual} is proposed. \
    
    It prevents the mutual transfer of erroneous knowledge by only passing high-confidence predictions to the other learning model. \
    However, such a mechanism can waste a large amount of unlabeled data during training. \

    \begin{figure}[h]
        \centering
        \includegraphics[height=3cm]{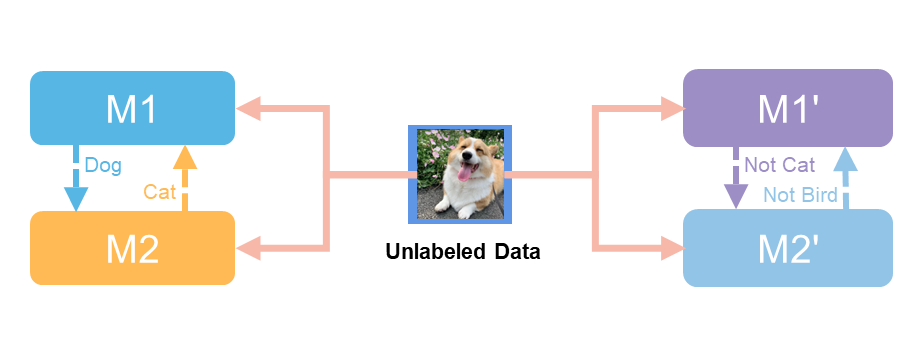}
        \caption{The dual model on the left side represents general mutual learning, i.e., the models pass strong information to each other such as information about the category with the highest prediction probability. The dual model near the right side exchanges weak information between each other, indicating which category the data does not belong to.}
        \label{intro}
    \end{figure}

    To address these problems, we propose a new semi-supervised classification framework based on dual pseudo-negative label learning.
    This framework comprises two submodels, and each submodel generates pseudo-negative labels as learning targets for the other submodel.
    Each submodel also provides pseudo-negative labels on augmented data for self-training. 
    The difference between our framework and general mutual learning is shown in Figure \ref{intro}. 
    We also propose a selection mechanism to identify the most representative pseudo-negative labels for the other model.
    The main contributions can be summarized as follows: 
    \begin{itemize}
        \item We propose a Dual Negative Label Learning (DNLL) framework, which not only improves the utilization of unlabeled data but also significantly reduces model parameter coupling compared to general mutual learning methods.
        \item We propose a selection mechanism to help select representative pseudo-negative labels and prove the effectiveness of this approach theoretically.
        \item We demonstrate the effectiveness of the proposed method experimentally on different benchmarks.
    \end{itemize}

    \section{Related Work}
    \subsection{Data Augmentation}
        Data augmentation plays a key role in model training, which is widely used in classification or segmentation.
        Data augmentation is used to expand the training set by applying random perturbations to improve algorithm performance and robustness. 
        Simple augmentation methods include random flips, horizontal or vertical transitions, geometric transformations, changing the contrast of images, and so on. 
        There are also complex operations.
        Mixup randomly selects two images and mixes them by a random proportion to expand the data set.
        The Cutout method replaces randomly selected image pixel values with zeros while leaving the labels unchanged \cite{devries2017improved}.
        In order to maximize the effect of data augmentation, strategies combining a range of augmentation techniques are proposed, such as AutoAugmentation \cite{yun2019cutmix}, RandAugmentation \cite{cubuk2020randaugment}, etc.
        We also employ data augmentation methods similar to other semi-supervised learning frameworks   \cite{berthelot2019mixmatch,berthelot2019remixmatch}.

    \subsection{Semi-Supervised Learning}
        Semi-supervised learning has received a lot of attention in recent years. 
        The main task of semi-supervised learning is to utilize labeled and unlabeled data to train algorithms.
        Many approaches based on consistency regularity, Pi-Model, Temporal Ensembling Model \cite{laine2016temporal}, Mean Teacher \cite{tarvainen2017mean}, Dual Student \cite{ke2019dual}, and so on. 
        Later, a series of holistic analysis methods, such as MixMatch \cite{berthelot2019mixmatch}, ReMixMatch \cite{berthelot2019remixmatch}, FixMatch \cite{sohn2020fixmatch}, have been proposed. 
       Alternatively, in   DMT, inconsistency between two models has also been used to exploit the correctness of pseudo-labels \cite{feng2022dmt}. 
        In this work, we propose an efficient semi-supervised classification framework with dual negative label learning.

    \subsection{Learning with Noisy Labels}
        In this case, models are trained with correctly labeled data and mistakenly labeled data.
        For example, based on the recent memory effect of a neural network, co-teaching \cite{han2018co} trains two models simultaneously, and each model can help the other one to filter out samples with large losses.
        Kim et al. \cite{kim2019nlnl} proposes a negative learning method for training convolutional neural networks with noisy data. This method provides feedback for input images about classes to that they do not belong.
        In this work, we propose to use low-confidence pseudo-labels as noisy labels for further learning.

    \subsection{Learning from Complementary Labels}
        A category corresponding to the complementary label is that a data item does not belong. 
        Due to difficulties in collecting labeled data, complementary-label learning is used in fully supervised learning methods \cite{ishida2017learning} and noisy-label learning methods \cite{kim2019nlnl}. 
        Complementary labels can be generated based on noisy labels \cite{ishida2017learning,kim2019nlnl}.
        In our method, complementary labels are generated based on model-generated pseudo labels.
    
\begin{figure*}[htbp]
    \centering
    \includegraphics[height=7cm,width=16cm]{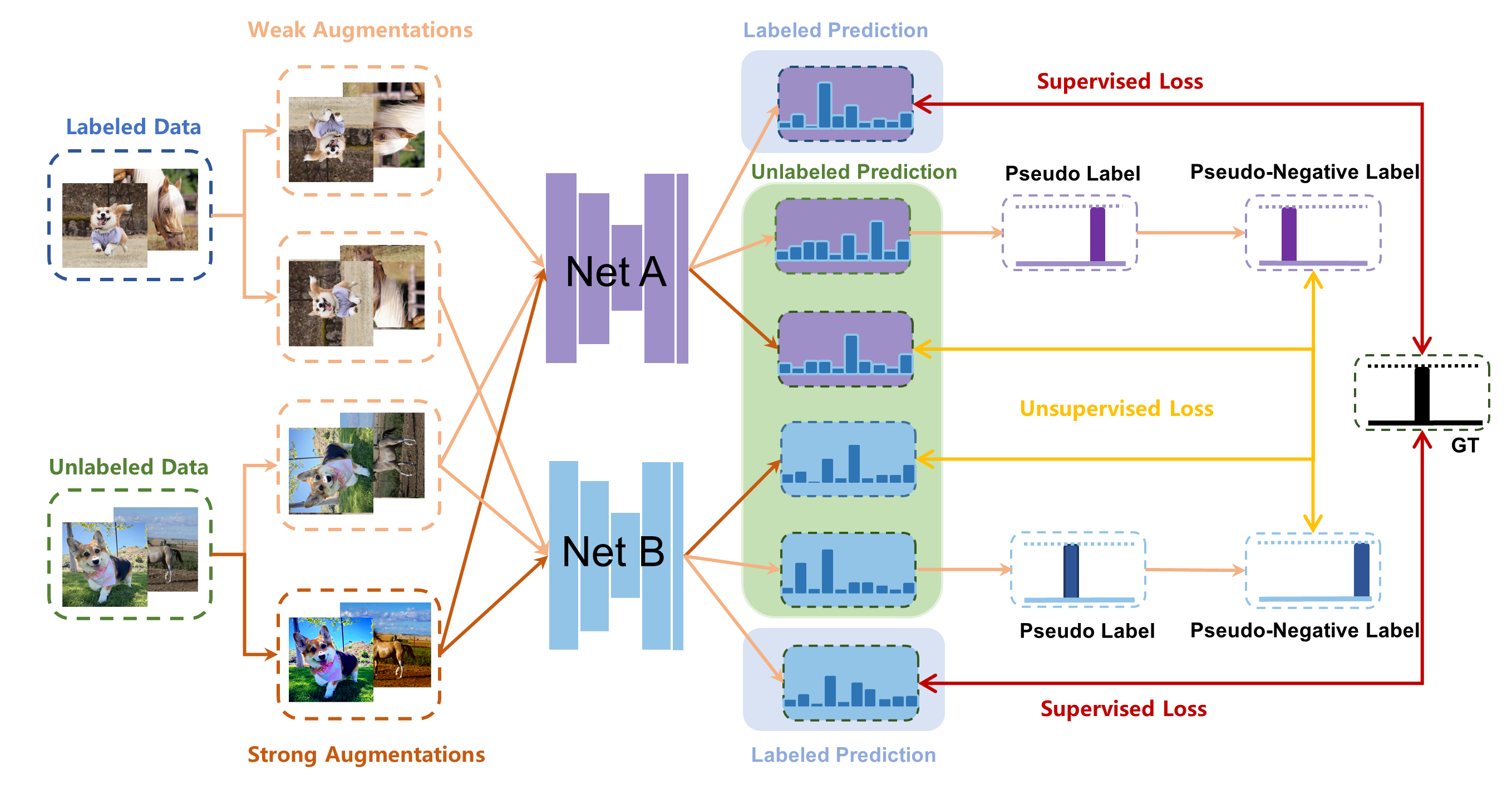}
    \caption{
    Overview of the DNLL framework. We use a small amount of labeled data and a large amount of unlabeled data to train a dual model. Each submodel within the dual model has the same structure and is initialized independently. For each labeled data, weak augmentations such as random cropping and random flipping are applied. A cross-entropy function is used to calculate the supervised loss. For each unlabeled data, besides weak augmentations, strong augmentations such as color jittering are applied. Each submodel generates pseudo-negative labels based on predictions of weakly augmented data, and these labels are used to teach the other submodels when predicting strongly augmented data.}
    
    \label{pipe}
\end{figure*}

\section{Methodology}
\subsection{Problem Definition}

    In traditional multi-model frameworks, learning models under-fitted in the early stage of training are likely to pass erroneous pseudo-labels to other models. Such errors can be accumulated and need to be filtered out. 
    In addition, consistency loss on the same erroneous pseudo-labels can also lead the multi-model framework to degenerate into a self-training model. 
 
    Therefore, in this section, we propose a multi-model semi-supervised learning framework to improve the utilization of unlabeled data and alleviate degeneration.
    We first describe the novel mutual learning framework called Dual Negative Label Learning. That detailed framework is shown in Figure \ref{pipe}, and then proposes an effective selection mechanism for choosing representative pseudo-negative labels. 

    In semi-supervised learning (SSL), the goal is to train a model by utilizing a small amount of labeled data and a large amount of unlabeled data. 
    Formally, we define a training set $D$ consisting of labeled data $D_l$=$\left\{ \left( X_i,Y_i  \right);i\in\left( 1,...,N  \right)\right\}$ 
    and  unlabeled data $D_u$=$\left\{ \left( X_j \right);j\in\left( 1,...,M  \right)\right\}$, and we use a dual model to allow each submodel learning from the other. The label $Y_i$ of the $i$-th data item is a one-hot vector.

\subsection{Supervised Learning}
        In supervised learning, labeled data are augmented by different weak augmentations for different submodels.
        \begin{align}
            X_{i}^{(1)}=&A_{w}^{(1)}(X_{i}),\\
            X_{i}^{(2)}=&A_{w}^{(2)}(X_{i}).
        \end{align}
        where $A_{w}^{(1)}, A_{w}^{(2)}$ denote different weak augmentation operations and $X_{i}^{(1)}, X_{i}^{(2)}$ denote weakly augmented data sets.
        
        We use the cross-entropy (CE) function for the supervised loss. In classification tasks, the image-level CE loss is as follows:
        \begin{equation}
            H(Y,\hat{Y})=-\sum_{i}^{}Y_i log(\hat{Y}_i)
        \end{equation}
        where $\hat{Y}$ is the predicted label, and $Y$ is the ground truth. 

        The supervised losses of the two submodels are as follows:
        \begin{align}
            \ell _{sup}^{(1)}= H(f_\theta (X_{i}^{(1)}), Y_i),\\
            \ell _{sup}^{(2)}=H(f_\varphi  (X_{i}^{(2)}), Y_i).
        \end{align}
        where $ f_\theta$ and $f_\varphi$ represent the operations of two submodels respectively, and $\theta$ and $\varphi$ represent parameters corresponding submodels.

\subsection{Unsupervised Learning}
    \subsubsection{Dual pseudo-negative label Learning}
        Most unsupervised learning parts in semi-supervised learning frameworks are realized by allowing each submodel to learn with pseudo-positive labels from other submodels.        
        To avoid model degeneration and error accumulation in this process, we propose a novel dual negative label learning approach. 
        In this approach, each submodel teaches the other that a given data item should not belong to a certain category. It allows model diversity and can reduce transferring of erroneous information.

        Pseudo-negative labels, namely, the labels that a corresponding data item does not belong to, are generated by taking complementary labels of the predicted label by a submodel. In our approach, we also select a few pseudo-negative labels as representative pseudo-negative labels.  %A representative pseudo-negative label is randomly selected from pseudo-negative labels. 
For data item $j$, its pseudo label $\hat{Y}_j$ and its representative pseudo-negative label $Y^{c}_j$ are randomly selected from all the candidates with equal probability (EP) as follows:
        \begin{align}
           \hat{Y}_j&=f(X_j),\\
            Y^{c}_j&\in z(f(X_j),m),
        \end{align}
        where $m$ is one by default, and $z$ is defined as follows:
        \begin{align}
            z(f(X_j),m) =& \{v| v\in \{0,1\}^K \text{, }\sum_i v_i = m, \nonumber\\
            & \text{ and } v[\arg\max \hat{Y}_j] \neq 1\}.
        \end{align}
        Here, $K$ is the number of categories, and $\{0,1\}^K $ represents a vector of length $K$ with elements equal to zero or one. 
        To increase the convergence rate, we can allow each submodel to generate multiple representative pseudo-negative labels for each weakly augmented data item for the other submodel to learn. Thus, $m$ can also be positive integers larger than one and less than $K$. 

        By teaching each other with pseudo-negative labels only, we reduce the coupling between submodels. 
        The loss function can be written as follows: 
        \begin{equation}
            L(\hat{Y},Y^{c})=-\sum_{i}^{}Y^{c}_{i}\log(1-\hat{Y}_{i})  \label{PNloss}
        \end{equation}
        where $\hat{Y}$ denotes the predictions from one submodel and $Y^{c}$ is the representative pseudo-negative labels from the other submodel.

        We also use weak and strong data augmentations for unlabeled data to improve the generalization ability of the dual model.
        The weak augmentations can be random cropping, random flipping, or simply outputting the original images. The strong augmentation operations can be color dithering or noise perturbations.
        Usually, predictions for weakly augmented data by a submodel will be more accurate than that for strongly augmented data.
        Thus, in our framework, the predictions of weakly augmented data by one submodel are used for generating pseudo-negative labels. We use these labels as learning targets for the other submodel feed by strongly augmented images. The augmentation process can be written as follows: 
        \begin{align}
            X_{j}^{(w)}=&A_{w}(X_{j}),\\
            X_{j}^{(s)}=&A_{s}(X_{j}),
        \end{align}
        where $A_{w}$ and $A_{s}$ denote the weak and strong augmentation operations, respectively. $X_{j}^{(w)}$ and $X_{j}^{(s)}$ denote the weakly and strongly augmented data items. 
        Consequently, we have 
        \begin{align}
            Y^{c_1}\in z(f_{\theta}(X_j^{(w)}),m) ,\\
            Y^{c_2}\in z(f_{\varphi}(X_j^{(w)}),m).
        \end{align}

        Therefore, the loss of learning between submodels is as follows:
        \begin{align}
            \ell _{cross}^{(1)}=L(f_{\theta}(X_j^{(s)}),Y^{c_2}),\\
            \ell _{cross}^{(2)}=L(f_{\varphi}(X_j^{(s)}),Y^{c_1}).
        \end{align}

        To further utilize the augmented data, we also developed a self-learning approach.
        In this approach, the generated pseudo-negative labels with weakly augmented data are also used by the same submodel to feed strong augmented data.
        The loss function can be written as follows: 
        \begin{align}
            \ell _{self}^{(1)}=L(f_{\theta}(X_j^{(s)}),Y^{c_1}),\\
            \ell _{self}^{(2)}=L(f_{\varphi}(X_j^{(s)}),Y^{c_2}).
        \end{align}
        
        The unsupervised loss of the dual model is a combination of the previous loss functions:
        \begin{align}
            \ell _{unsup}^{(1)}=\ell _{corss}^{(1)}+\ell _{self}^{(1)},\\ 
            \ell _{unsup}^{(2)}=\ell _{corss}^{(2)}+\ell _{self}^{(2)}.
        \end{align}
        The final total loss of the dual model in the DNLL is a combination of the supervised loss and the unsupervised one, as follows:
        \begin{align}
            \ell^{(1)}=\ell _{sup}^{(1)}+\lambda \ell _{unsup}^{(1)},\\ 
            \ell^{(2)}=\ell _{sup}^{(2)}+\lambda \ell _{unsup}^{(2)},
        \end{align}
        where $\lambda $ is a hyperparameter to balance the supervised loss item and the unsupervised loss item. 
        The complete algorithm is shown in Algorithm \ref{alg:algorithm}.

        From this pseudo code, we can see that the running time is proportional to the size of the input data. 
        If the size of unlabeled data, $M$, is much larger than that of the labeled data, $N$, which usually happens in semi-supervised learning, the running time is approximately proportional to the size of the unlabeled data. 
        Thus, the time complexity is $O(M)$.

\subsubsection{Error Perception Mechanism for Selecting Pseudo-Negative Labels}
    In the above section, for an unlabeled data item, a representative pseudo-negative label is randomly selected from all the candidates with equal probability. 
    To incorporate the performance of each submodel in different categories, we propose an Error Perception Mechanism (EPM). 

\begin{figure}[htbp]
    \centering
    \includegraphics[height=4cm]{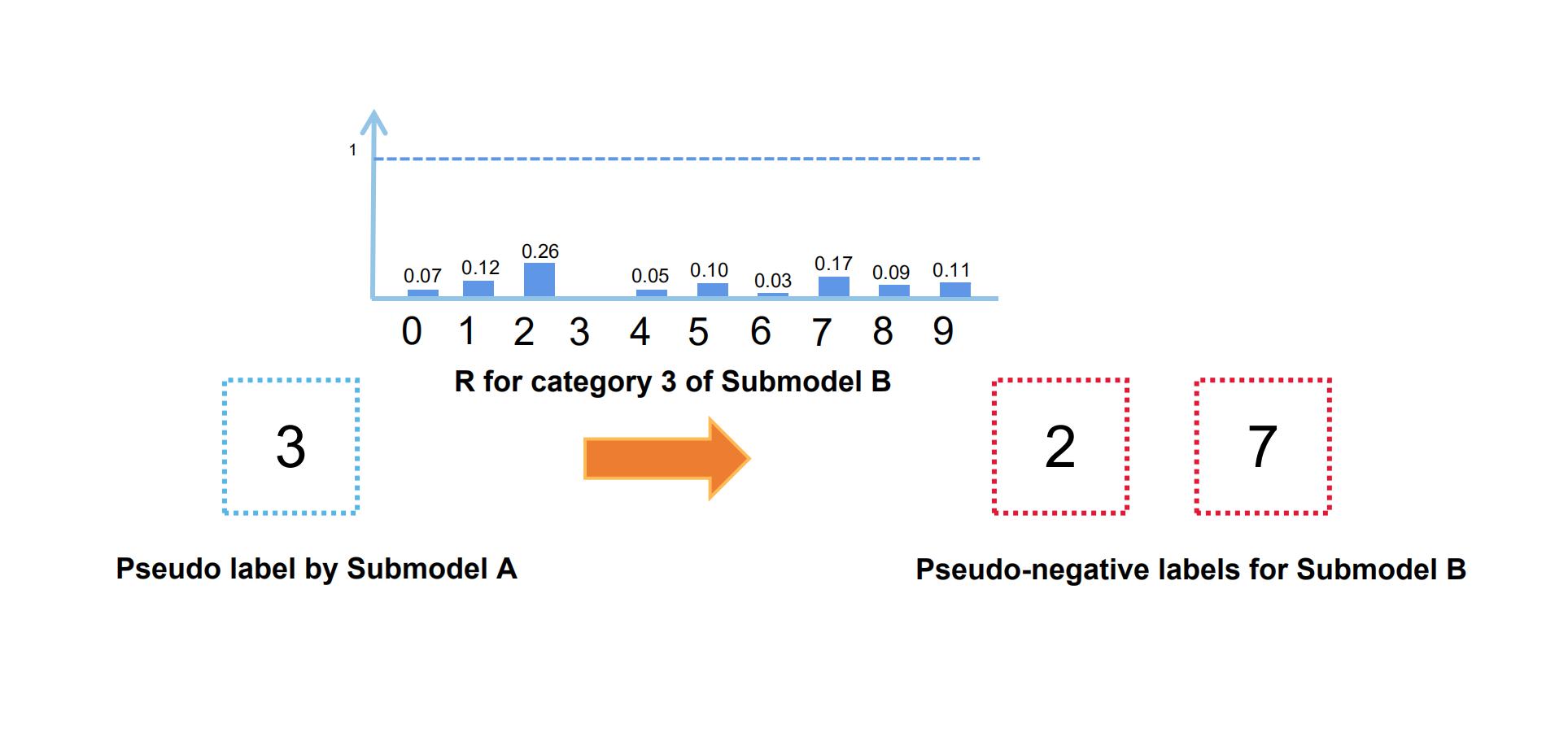}
    \caption{The generating process of pseudo-negative labels. For an unlabeled data item, a submodel makes a prediction to generate a pseudo label (3 in this example) and then randomly selects two pseudo-negative labels according to  $ R$ of the other submodel. }
    \label{pb}
\end{figure}
    In this approach, for a given data item, if a submodel is prone to misclassify it into the other category, the pseudo-negative label generated by the other submodel should include that misclassified category. 
    Therefore,  we compute the probability of misclassification for each category of each submodel using labeled data.   
    Formally, for a submodel, we define a vector $Pr_k$ for category $k$ with its $i$-th element defined as follows: 

    \begin{equation}
        Pr_k[i]=
        \begin{cases}
        \sum_{j=1}^{N_{k}}p_{ij},& \text{ $i\ne k$}\\
        0& \text{ $i=k$}
        \end{cases}
    \end{equation}
    where $N_{k}$ denotes the total number of data with category $k$ being misclassified into category $i$,
    and $p_{ij}$ represents the confidence that the $j$-th misclassified sample belongs to the $i$-th category. 
    We may also use EMA  to update $Pr_k$ for stability.
   
    It is then normalized with a softmax function.   
    \begin{equation}
        R_{k} = Softmax(Pr_{k}).
    \end{equation}
    We use superscripts to represent submodels, so $R_{k}^{(1)}$ and $R_{k}^{(2)}$ are misclassification probabilities for the first and the second submodels. An example of the $R_{k}$-based pseudo-negative label generation process is shown in Figure \ref{pb}.

    Therefore, when computing $\ell _{cross}^{2}$, we sample $Y^{c_1}$  from  $z(f_\theta(X_j^{(w)}),m)$  such that the probability that $Y^{c_2}_j[k]=1$ is proportional to $R_k^{(2)}$. A similar approach applies when computing   $\ell _{cross}^{1}$.
    
    \begin{algorithm}[tb]
        \caption{Pseudo code for the training process of DNLL.}
        \label{alg:algorithm}
        \textbf{Input}:The labeled dataset  $D_l$=$\left\{ \left( X_i,Y_i  \right);i\in\left( 1,...,N  \right)\right\}$ and the unlabeled dataset $D_u$=$\left\{ \left( X_j \right);j\in\left( 1,...,M  \right)\right\}$.

        The two submodels are $f_{\theta}$ and $f_{\varphi}$.
        \begin{algorithmic}[1] 
  
        \FOR{each epoch}
            \FOR{each batch}
                \STATE $(\chi_l,Y_l):\rm{select\ a\ batch\ of\ data\ from}\ \mit{D_l}$
                \STATE $(\chi_u):\rm{select\ a\ batch\ of\ data\ from}\ \mit{D_u}$
                \STATE $\chi^{(1)}_{l} = A^{(1)}_{w}(\chi_l)$
                \STATE $\chi^{(2)}_{l} = A^{(2)}_{w}(\chi_l)$
                \STATE $\chi^{(w)}_{u} = A_{w}(\chi_{u})$
                \STATE $\chi^{(s)}_{u} = A_{s}(\chi_{u})$
                \STATE $\ell^{(1)}_{sup} = H(f_{\theta}(\chi^{(1)}_{l}),Y_{l})$
                \STATE $\ell^{(2)}_{sup} = H(f_{\varphi}(\chi^{(2)}_{l}),Y_{l})$
                \STATE $Y^{c_1}\in z(f_{\theta}(\chi_u^{(w)}),m)$ 
                \STATE $Y^{c_2}\in z(f_{\varphi}(\chi_u^{(w)}),m)$
                \STATE $\ell^{(1)}_{unsup}=L(f_{\theta}(\chi^{(s)}_{u}),Y^{c2})$
                \STATE $\ell^{(2)}_{unsup}=L(f_{\varphi}(\chi^{(s)}_{u}),Y^{c1})$
                \STATE $f_{\theta}=\arg \min_{f_{\theta}}(\ell^{(1)}_{sup}+\lambda \ell^{(1)}_{unsup})$
                \STATE $f_{\varphi}=\arg \min_{f_{\varphi}}(\ell^{(2)}_{sup}+\lambda \ell^{(2)}_{unsup})$
            \ENDFOR\\
        \ENDFOR\\
        \textbf{return} $f_{\theta}, f_{\varphi}$
        \end{algorithmic}
    \end{algorithm}

\subsection{Theoretical Analysis}
First, we demonstrate that in the mutual learning framework based on a dual model, passing pseudo-negative labels between submodels is less likely to have error accumulation than that of passing pseudo labels, especially at the early stages of training.

\newtheorem{thm}{\bf Theorem}[section]
\begin{thm}\label{thm1}
    The error rate (ER) for transferring pseudo-negative labels from one submodel to the other is expected to be $\frac{m}{K-1}$ of the ER when transferring pseudo labels, where $m$ is the number of selected pseudo-negative labels and $K$ is the number of categories for each data item.
\end{thm}
\begin{proof}
    Without loss of generality, we define that the prediction accuracy of one submodel is $q$ for unlabeled data. 
    Therefore, when transferring pseudo labels, the probability that that submodel provides correct learning targets to the other is $q$.

    When transferring $m$ pseudo-negative labels, if the submodel predicts correctly, it transfers correct negative labels. If  the submodel predicts mistakenly, the chance of providing correct negative labels is
\begin{align}
\frac{C^{m}_{K-2}}{C^{m}_{K-1}} ,
    \end{align}
    where ${C^{m}_{K-1}}$ denotes the total number of combinations of selecting $m$ pseudo-negative labels from all the $K-1$ pseudo-negative labels, and ${C^{m}_{K-2}}$ denotes the number of combinations of selecting $m$  pseudo-negative labels from $K-2$ truly negative labels. $K-2$ is obtained by taking all the $K$ categories except two categories corresponding to one pseudo label and one ground-truth label. 
    Therefore,  the probability of providing the correct learning target is 
    \begin{align}
    q+(1 - q)\frac{C^{m}_{K-2}}{C^{m}_{K-1}}=1-\frac{(1-q)m}{K-1} .
    \end{align}
    Therefore, the error rate of transferring pseudo-negative labels is 
    \begin{equation}
            1- (1-\frac{(1-q)m}{K-1})=  (1-q)\frac{m}{K-1}.
    \end{equation}
    As the error rate of transferring pseudo labels is   $1-q$,  the error rate of transferring pseudo-negative labels is $ \frac{m}{K-1}$ of that of transferring pseudo labels. 
    Therefore, transferring pseudo labels can provide a better learning target, and a smaller $m$ and a larger $K$ can further reduce the error accumulation.
\end{proof}

For two submodels with the same structure, when they are converged to be the same, they can no longer be used for semi-supervised learning. 
We need to avoid such scenarios, especially in the early training stages. In the unsupervised learning part, we demonstrate that when transferring knowledge with pseudo-negative labels, it is unlikely to have two submodels degenerate into the same.
 \begin{thm}\label{thm1}
    When transferring representative pseudo-negative labels randomly, the probability that two submodels are optimized for different objectives is  $1-\frac{\sqrt{2\pi m}}{eK}(\frac{m}{K})^m$ approximately, where $m$ is the number of representative pseudo-negative labels and $K$ is the number of categories.  
\end{thm}
\begin{proof}
    Without loss of generality, we assume that two submodels produce the same prediction with probability $q$ and    when they produce the same pseudo labels, the probability that the two submodels can produce the same representative pseudo-negative labels is 
\begin{align}\frac{1}{C_{K-1}^m}.
\end{align}
Similarly, the probability that two submodels produce different predictions is $1-q$, and when they produce different predictions, the probability that they produce the same pseudo labels is 
 \begin{align}
 \frac{1}{C_{K-2}^m}. 
 \end{align}

 Thus, the probability that the two submodels transferring the same representative pseudo-negative label is 
 \begin{align}
    &\frac{q}{C_{K-1}^m}+\frac{1-q}{C_{K-2}^m}\\
    =&\frac{m!(K-2-m)!(K-1-qm)}{(K-1)!}\\
    \approx  & (K-1-qm) \times\nonumber \\ 
    &\frac{\sqrt{2\pi m}(\frac{m}{e})^{m} \sqrt{2\pi(K-2-m)}(\frac{K-2-m}{e})^{K-2-m}}{\sqrt{2\pi(K-1)}(\frac{K-1}{e})^{K-1}}\label{th2app1} \\ 
    \approx& %\frac{m!}{K^m}
    \frac{\sqrt{2\pi m}}{eK}(\frac{m}{K})^m \label{th2app2}
\end{align} 
where the approximation in Eq. (\ref{th2app1}) is obtained by the Stirling's approximation, and that in  Eq. (\ref{th2app2}) is by considering $K>>m$.

\end{proof}

\section{Experiments}
    In this section, we first introduce benchmarks used in experiments and briefly describe the details of the experiments. 
    Then we compare DNLL with other methods.
    Finally, we evaluate the efficiency of DNLL from different perspectives.
    
    \subsection{Benchmark datasets}
    
        In the classification task, we use the public benchmark datasets CIFAR-10 \cite{krizhevsky2009learning}, SVHN \cite{netzer2011reading}, and MNIST as many others.
        The CIFAR-10 dataset includes 50000 training images and 10000 test images, and the total number of categories is ten. 
        We randomly select 500 images for each category as the validation set.
        The total number of categories of SVHN Dataset is ten, in which the training set contains 73257 images and the test set contains 26032 images. 
        We also randomly select 500 images for each category as the validation set. 
        The MNIST dataset includes 60000 training images and 10000 test images, and the total number of categories also is ten.
        We randomly select 50 images for each category as the validation set.

    \subsection{Implementation Details}
        
        Our approach is implemented on Pytorch.
        For the training stage, the following configurations are used.
        The learning rate is 0.03, and the weight decay is $5\times10^{-4}$. 
        The momentum is 0.9. We use the cosine annealing technique with batch size 256.
        We report performances on the test set averaged from three runnings. 
        For dual models, we use WideResNet-28-2 (WRN-28-2)\cite{zagoruyko2016wide} and 13-layer CNN as other approaches
        \cite{berthelot2019mixmatch,ke2019dual}.

        We use data augmentation techniques in our experiments. 
        The data augmentation operation for each data set is performed exactly following its corresponding literature for fairness. 
        Specifically, for the MNIST dataset, we do not change the input data \cite{luo2018smooth}. 
        For the CIFAR-10 dataset, when using the 13-layer CNN as the model \cite{ke2019dual}, we make the original image as a weakly augmented version and the noise-processed image as a strongly augmented version. 
        When using WideResNet-28-2 as the model \cite{feng2022dmt}, the weak augmentation operations we used include random cropping and random flipping, and the strong augmentation operation is random color jittering. 
        For the SVHN dataset \cite{laine2016temporal}, we only use the horizontal translation as the strong augmentation operation and the original image as the weakly augmented version.
        
    \subsection{Comparison on Benchmarks}
        In experiments with the CIFAR-10 dataset, we randomly select $1$K, $2$K, and $4$K data items, respectively, as labeled data and the rest as unlabeled data.
        
        We compare our method with others: \emph{$\Pi$} model, Temporal Ensembling \cite{laine2016temporal}, VAT\cite{miyato2018virtual} and Mean Teacher \cite{tarvainen2017mean} based on consistency regularization;
        \emph{$\Pi$}+STNG \cite{luo2018smooth}, LP+SSDL and LP-SSDL-MT \cite{iscen2019label} based on graph methods; Filtering CCL, Temperature CCL \cite{li2019certainty}, TSSDL, TSSDL-MT \cite{shi2018transductive} and TNAR-VAE \cite{yu2019tangent} based on mean-teacher frameworks; Curriculum Labeling (CL) \cite{cascante2020curriculum} based self-training; MixMatch \cite{berthelot2019mixmatch} based on strong hybrid method. 
        We also compare our approach with others based on dual models: Deep Co-Training (DCT) \cite{qiao2018deep}, Dual student(DS) \cite{ke2019dual}, Mutual Learning of Complementary Networks(CCN) \cite{wu2019mutual} and Dynamic Mutual Training (DMT) \cite{feng2022dmt}. 
        The symbol $^\dag$ indicates that the results are reported in \cite{chen2020semi} and \cite{hu2021simple}.
        The symbol '-' indicates that the corresponding results have not been reported in this literature.

        From Table \ref{tab:1} and Table \ref{tab:2}, we can find that our method performs relatively well with 1k labels and outperforms all the other methods in other cases. 
        From Table \ref{tab:1}, the accuracy of our approach ranges between 87.87\% and 92.06\%, which outperforms most of the other methods using the dual model, i.e., DCT, Dual Student, and CCN.
        From Table \ref{tab:2}, the MixMatch is 0.53\% lower than our approach at the accuracy with 4K labels.  
        The DMT is 0.41\% and 0.08\% lower than our approach at the accuracy with 1K and 4K labels, respectively.
        Figure \ref{plot} demonstrates the performance of DNLL during the training process on the test set. As the epoch number increases, the training accuracy increases.

    \begin{table}[htbp]
        \centering
        \caption{Accuracy on the Test Set of CIFAR-10 with the 13-layer CNN as the backbone.}
        \label{tab:1}
            \begin{tabular}{lllll}
            \hline\noalign{\smallskip}
            Method & 1K  & 2K  & 4K \\
            \noalign{\smallskip}\hline\noalign{\smallskip}
            \emph{$\Pi$} model$^\dag$ & 68.35 & 82.43 & 87.64 \\
            Temporal ensembling$^\dag$ & 76.69 & 84.36 & 87.84\\
            Mean Teacher & 81.78 & 85.67 & 88.59 \\
            \emph{$\Pi$}+SNTG$^\dag$ & 78.77 & 85.35 & 88.64 \\
            LP-SSDL$^\dag$ & 77.98 & 84.34 & 87.31\\
            LP-SSDL-MT$^\dag$ & 83.07 & 86.78 & 89.39 \\
            Filtering CCL$^\dag$ & 81.78 & 85.67 & 88.59\\
            Temperature CCL$^\dag$ & 83.01 & 87.43 & 89.37 \\
            TSSDL$^\dag$ & 78.87 & 85.35 & 89.10 \\
            TSSDL-MT$^\dag$ & 81.59 & 86.46 & 90.70\\
            TNAR-VAE$^\dag$ & - & - & 91.15 \\
            DCT & - & - & 90.97  \\
            Dual Student & 85.83 & 89.28 & 91.11\\
            CCN & \textbf{87.95} & 89.63 & 91.2 \\
            \noalign{\smallskip}\hline\noalign{\smallskip}
            DNLL (Ours) & 87.87 & \textbf{90.65} & \textbf{92.06}\\
            \noalign{\smallskip}\hline
            \end{tabular}
    \end{table}
        
    \begin{table}[htbp]
        \centering
        \caption{Accuracy on the Test Set of CIFAR-10 with the WRN-28-2 as the backbone.}
        \label{tab:2}
            \begin{tabular}{lllll}
            \hline\noalign{\smallskip}
            Method & 1K  & 4K \\
            \noalign{\smallskip}\hline\noalign{\smallskip}
            VAT$^\dag$ & 81.36  & 88.95 \\
            Mean Teacher$^\dag$ & 82.68  & 89.64\\
            CL & 90.61  & 94.02 \\
            MixMatch & \textbf{92.25}  & 93.76\\
            DMT & 91.51  & 94.21\\
            \noalign{\smallskip}\hline\noalign{\smallskip}
            DNLL (Ours) & 92.03 & \textbf{94.29} \\
            \noalign{\smallskip}\hline
        \end{tabular}
    \end{table}

    \begin{table}[htbp]
        \centering
        \caption{Accuracy on the Test Set of SVHN with the WRN-28-2 as the backbone.}
        \label{tab:3}    
        \begin{tabular}{lll}
        \hline\noalign{\smallskip}
        Method & 1K  & 4K \\
        \noalign{\smallskip}\hline\noalign{\smallskip}
        Pseudo-Labeling & 90.06 & - \\
        \emph{$\Pi$} model & 92.46 & - \\
        VAT$^\dag$ & 94.02 & 95.80 \\
        Mean Teacher$^\dag$ & 96.25 & 96.61 \\
        \noalign{\smallskip}\hline\noalign{\smallskip}
        DNLL (Ours) & \textbf{96.41} & \textbf{96.84} \\
        \noalign{\smallskip}\hline
        \end{tabular}
    \end{table}

    \begin{table}[htbp]
        \centering
        \caption{Accuracy on the Test Set of MNIST with the 13-layer CNN as the backbone.}
        \label{tab:4}
            \begin{tabular}{lllll}
            \hline\noalign{\smallskip}
            Method &  20  & 50  & 100 \\
            \noalign{\smallskip}\hline\noalign{\smallskip}
            ImprovedGAN$^\dag$ & 83.23  & 97.79 & 99.07\\ 
            Triple GAN$^\dag$ & 95.19  & 98.44 & 99.09\\
            \emph{$\Pi$} model$^\dag$ & 93.68  & 98.98 & 99.11\\
            \emph{$\Pi$} + SNTG$^\dag$ & 98.64  & 99.06  & 93.34\\
            \noalign{\smallskip}\hline\noalign{\smallskip}
            DNLL (Ours) & \textbf{99.19} & \textbf{99.32}  & \textbf{99.54}\\
            \noalign{\smallskip}\hline
        \end{tabular}
    \end{table}

    \begin{figure}[h]
        \centering
        \includegraphics[height=5cm]{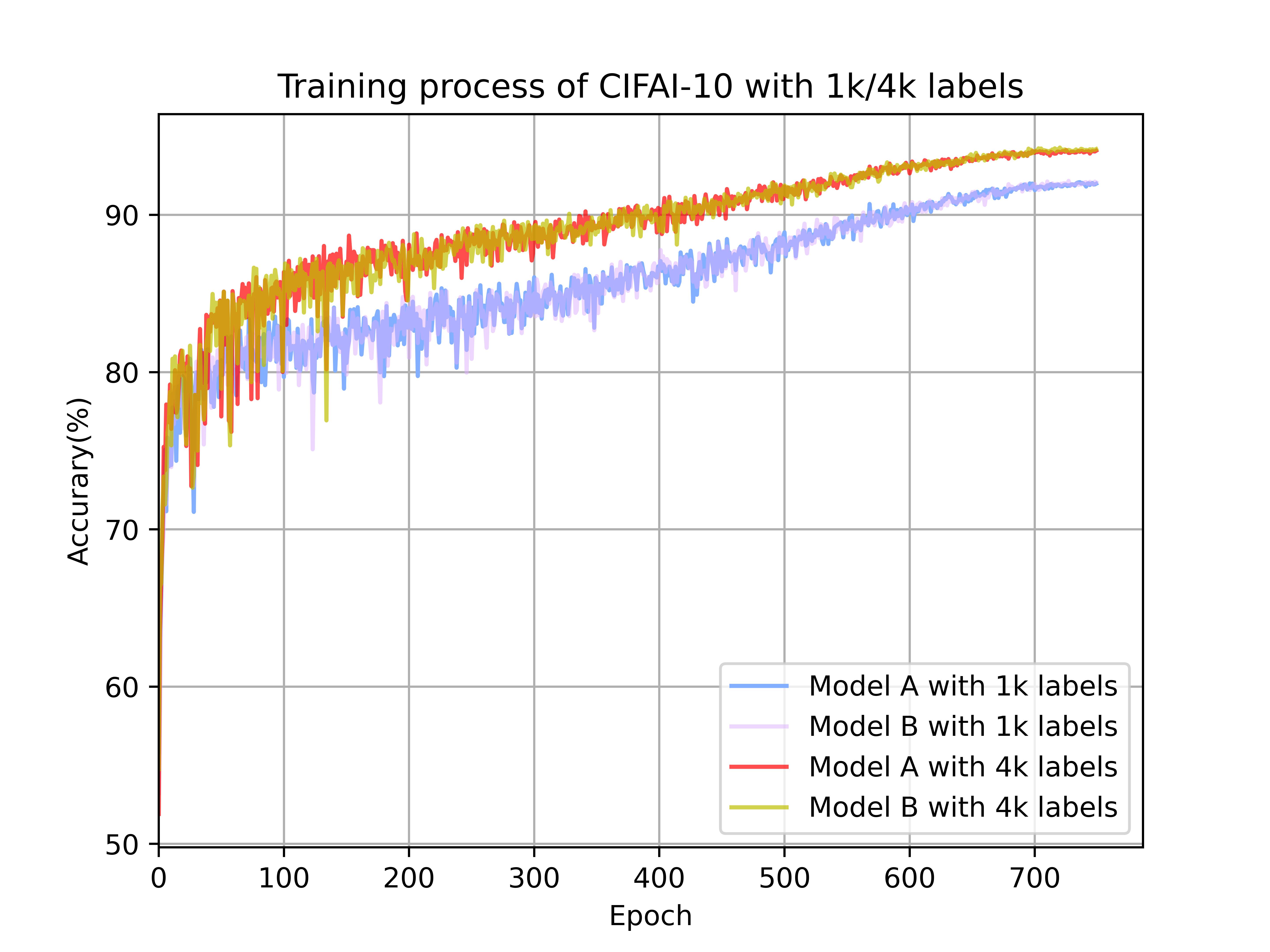}
        \caption{Performance of DNLL on the test set during training with the CIFAR-10 dataset of 1000 and 4000 labeled data.}
        \label{plot}
    \end{figure}

        In the SVHN dataset, 1K and 4K items are also randomly selected as labeled data. 
        We compare our method with others as follows: \emph{$\Pi$} model \cite{laine2016temporal}, Pseudo-Labeling \cite{lee2013pseudo}, VAT \cite{miyato2018virtual} and Mean Teacher \cite{tarvainen2017mean}.
        The symbol $^\dag$ indicates that the results are reported in \cite{hu2021simple}.
        All the approaches use WideResNet-28-2 as the backbone model. 
        As shown in Table \ref{tab:3}, our method outperforms all the other approaches.  

        For the MNIST dataset, 20, 50, and 100 data items are randomly selected as labeled data.
        We compare the DNLL with other semi-supervised methods, i.e., ImprovedGAN \cite{salimans2016improved}, Triple GAN \cite{li2017triple}, \emph{$\Pi$} model \cite{laine2016temporal} and \emph{$\Pi$} + STNG \cite{luo2018smooth}.
        The symbol $^\dag$ indicates that the results are reported in \cite{luo2018smooth}.
        All the above methods use the 13-layer CNN as the model.
        As shown in Table \ref{tab:4}, the DNLL outperforms the other approaches.

    \subsection{Sensitivity Analysis}
        We conduct a sensitivity analysis on the CIFAR-10 dataset with 4K labeled data items to analyze the relationship between representative pseudo-negative label number $m$ and the accuracy of the model under different selection mechanisms that were introduced in the methodology section: Equal Probability (EP) vs. Error Perception Mechanism (EPM).
        As the number of representative pseudo-negative labels tends to be less than half of the total number of categories, here we compare with $m \leq 4$. 
        From Table. \ref{tab:5}, we can find that generally, the error perception mechanism performs better than selecting with equal probability, and moderately increasing $m$ is helpful to increase the performance. When $m$ is too large, for example, close to half of the total number of categories, pseudo labels are likely to be selected, and the performance can be undermined.

        \begin{table}[htbp]
            \caption{Accuracy under different choices of $m$ and different selection mechanisms for representative pseudo-negative labels.} 
            \label{tab:5}    
            \begin{tabular}{l|c|c|c|c}
            \hline\noalign{\smallskip}
            Selection Method & $m = 1$ & $m = 2$ & $m = 3$ &$m = 4$\\
            \noalign{\smallskip}\hline\noalign{\smallskip}
            {EP} & 92.9 & 93.76 & \textbf{94.01} & 93.78\\
           {EPM} & 93.12 & 93.84 & \textbf{94.29} & 93.77 \\
            \noalign{\smallskip}\hline
            \end{tabular}
        \end{table}

    \subsection{Comparison with variants of DNLL}
        In this part, we demonstrate that using mutual learning framework in DNLL is more efficient compared to a self-learning framework.
        We compare the performance of these two learning frameworks. We can see from Table. \ref{tab:6} that the mutual learning framework under the dual model is better. 
        This is mainly because erroneous information can be filtered out by each other with different capabilities, avoiding the accumulation of errors.

    \begin{table}[htbp]
        \centering
        \caption{Comparison of the performance of mutual learning (ML) and self-learning (SL) with DNLL.}
        \label{tab:6}      
        \begin{tabular}{lllll}
        \hline\noalign{\smallskip}
        Method & 4k labels \\
        \noalign{\smallskip}\hline\noalign{\smallskip}
        SL w/o EPM & 92.78 \\
        SL & \textbf{93.03} \\
        ML w/o EPM & 94.01\\
        ML & \textbf{94.29}\\
        \noalign{\smallskip}\hline
        \end{tabular}
        \end{table}

    \subsection{Visualization of embeddings}
        We conduct experiments on MNIST with 20 labels without augmentation \cite{luo2018smooth}. 
        We visualize the embeddings of DNLL and a fully supervised learning method, respectively, on testing data under the same settings. 
            We use t-SNE \cite{van2008visualizing} to project the representations of the last hidden layer into two dimensions.
            Figure \ref{visualization} shows the results. Each point corresponds to an item in the testing set, and different ground-truth classes are encoded with different colors. 
            It demonstrates that the representations obtained from DNLL can better identify each class in the embedding space.

            \begin{figure}[htbp]
                \begin{minipage}[t]{0.5\linewidth}
                    \centering
                    \includegraphics[width=\textwidth, height=4cm]{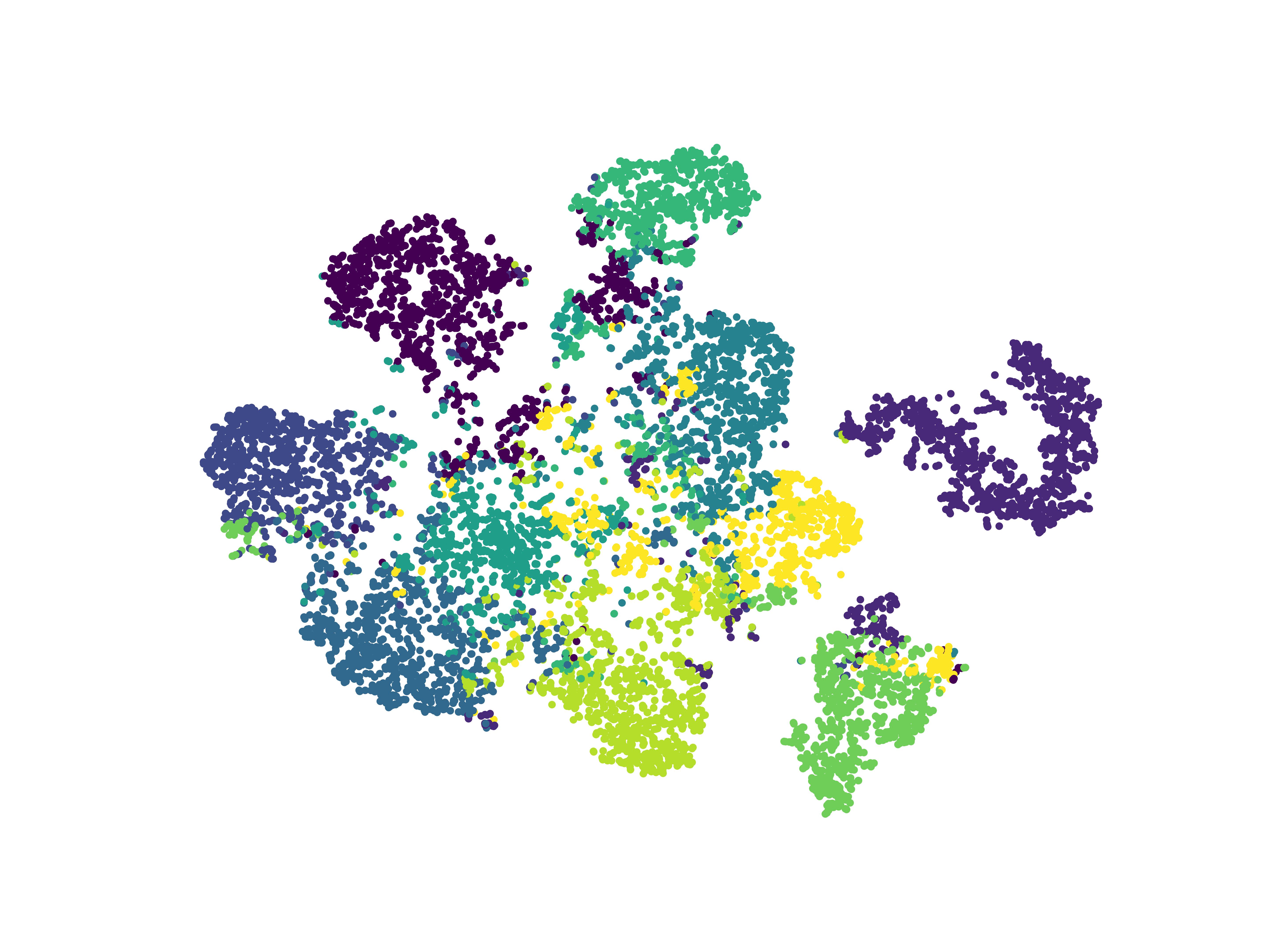}
                    \centerline{}
                \end{minipage}%
                \begin{minipage}[t]{0.5\linewidth}
                    \centering
                    \includegraphics[width=\textwidth, height=4cm]{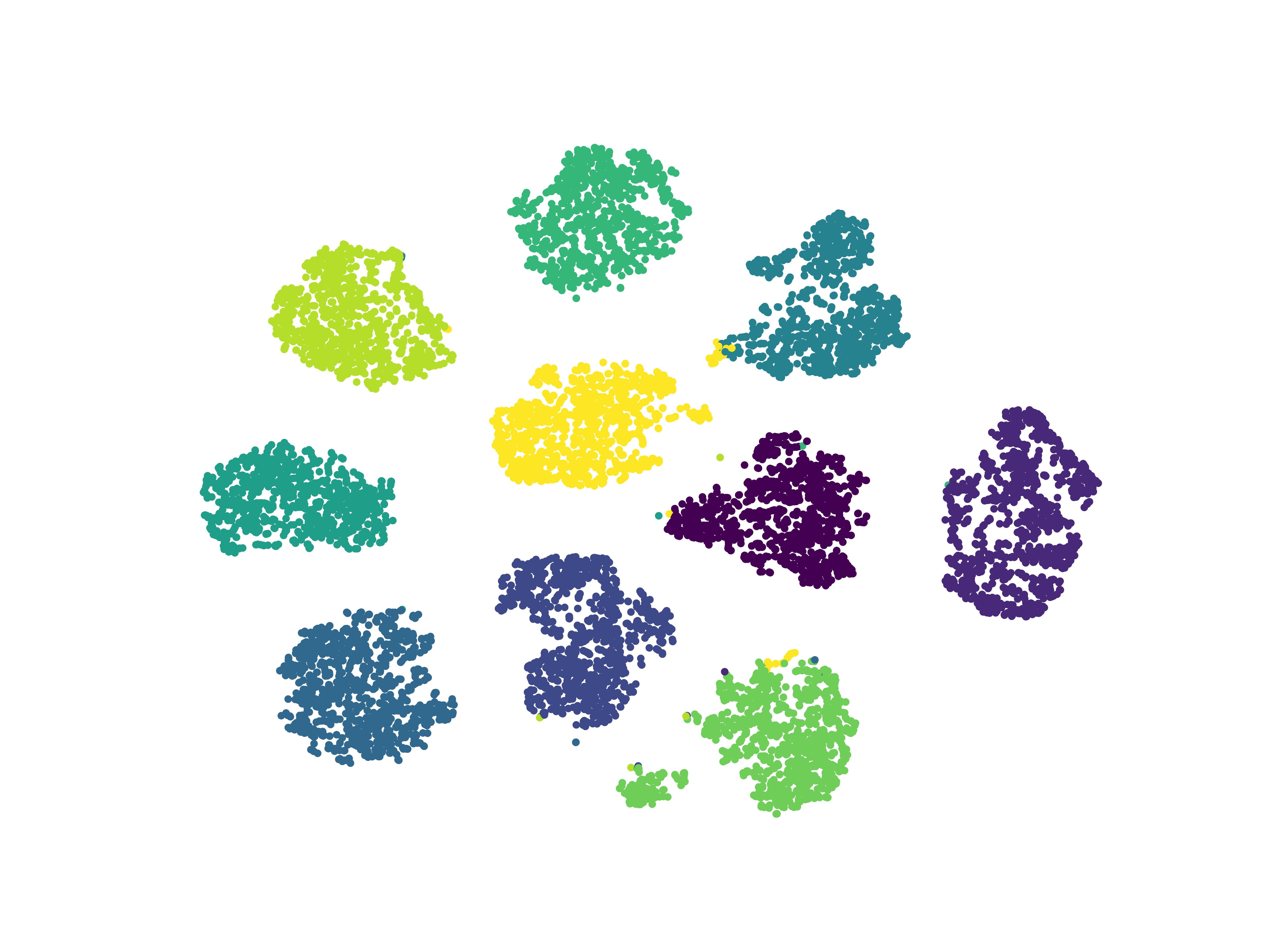}
                    \centerline{}
                \end{minipage}
                \caption{The t-SNE plot of the last hidden layer on the test data of MNIST with 20 labels: the baseline model (left) and our model (right). Our model can learn more discriminative representation.}
                \label{visualization}
            \end{figure}

    \subsection{Generalizability  of DNLL}

        To verify the generalizability of DNLL, we combine the ideology of DNLL method with the Dual Student method and the Mean Teacher method.
        For Dual Student, we use DNLL on the unstable samples discarded by the Dual Student. 
        As can be observed from the left side of Figure \ref{scalability}, our approach can take advantage of the discarded unlabeled data, which in turn improves the overall performance.
        In addition, we combine DNLL with Mean Teacher to use all the unlabeled data together. 
        From the right side of Figure \ref{scalability}, we can see that DNLL contributes significantly to the overall performance improvement.
        These experiments demonstrate that DNLL can be used in combination with other semi-supervised methods to jointly improve model performance.

        \begin{figure}[htbp]
            \begin{minipage}[t]{0.5\linewidth}
                \centering
                \includegraphics[width=\textwidth, height=3cm]{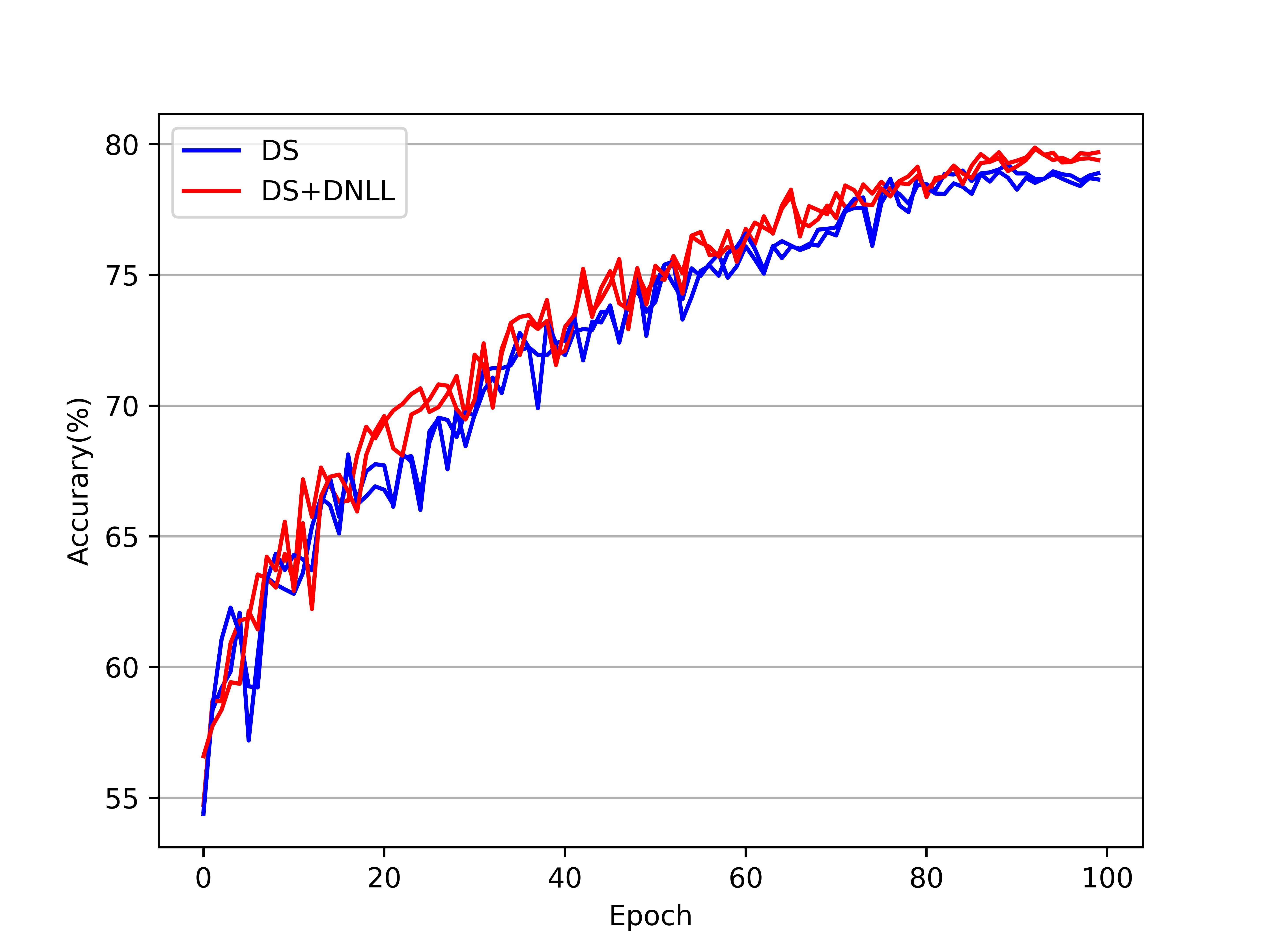}
                \centerline{}
            \end{minipage}%
            \begin{minipage}[t]{0.5\linewidth}
                \centering
                \includegraphics[width=\textwidth, height=3cm]{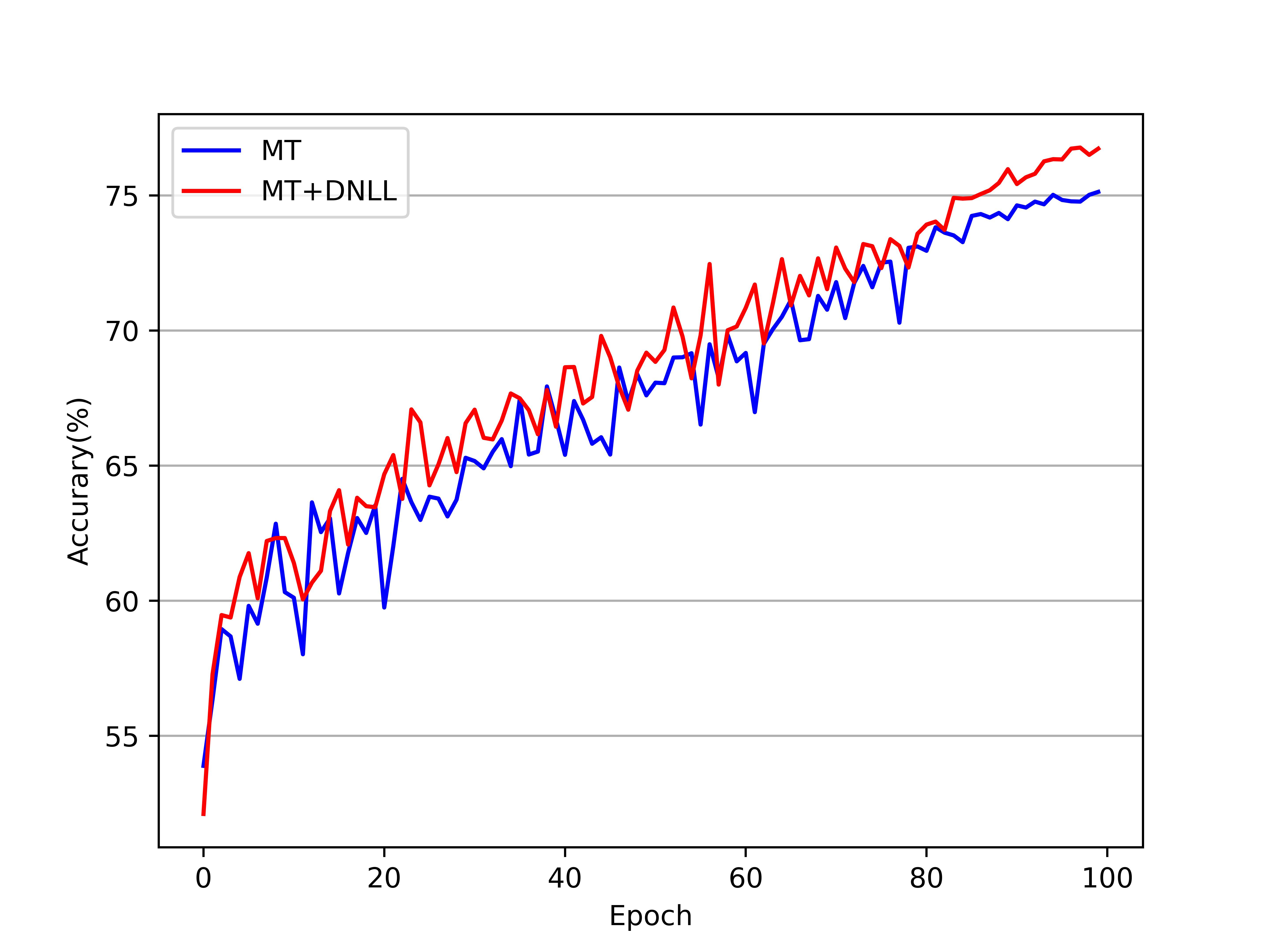}
                \centerline{}
            \end{minipage}
            \caption{The left side of the above figure shows the iteration process of combining DNLL and Dual Student. The right side shows the training process of combining DNLL and Mean Teacher.}
            \label{scalability}
        \end{figure}

    \subsection{Domain Adaptation using DNLL}

         \begin{figure}[h]
            \centering
            \includegraphics[height=5cm, width=8cm]{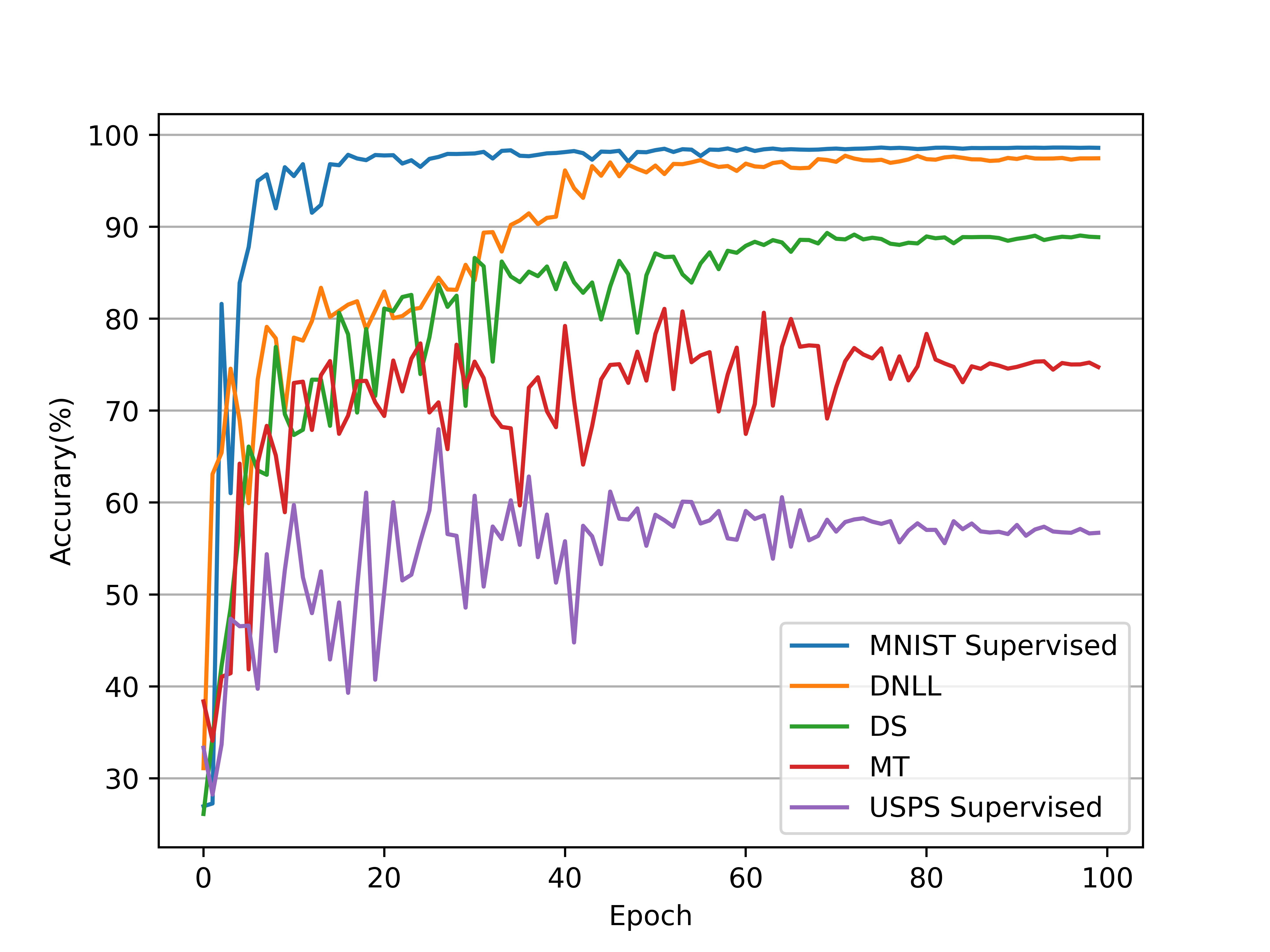}
            \caption{Test curves of domain adaptation from USPS to MNIST versus the number of epochs. The DNLL avoids overfitting and improves the result remarkably.}
            \label{cd}
        \end{figure} 
        
        Domain adaptation is the closely related to semi-supervised learning. It aims at knowledge transfer from the source domain to the target domain.
        Zhan et al. \cite{ke2019dual} propose Dual Student method to overcome the shortcomings of Mean Teacher and demonstrate the effectiveness of a dual model in domain adaptation tasks.
        In this section, we use DNLL for adapting digital pattern recognition from USPS to MNIST.
        We use USPS as the source domain and MNIST as the target domain and show that the DNLL has advantages over the Dual Student and Mean Teacher.

        USPS and MNIST are both grayscale hand-written digital datasets, the difference is that the image size is 16x16 for USPS and 28x28 for MNIST.
        The training set of USPS contains 7291 images, and the training set of MNIST contains 60,000 images. And the test set for the experiments uses the MNIST test set containing 10,000 images.
        We compare DNLL with Dual Student, Mean Teacher, fully supervised learning for the source domain and fully supervised learning for the target domain with 7k balanced labels.
        Following experiment settings in Dual Student \cite{ke2019dual}, we use cubic spline interpolation to match the resolution between the two dataset images and employ a 3-layer CNN \cite{ke2019dual} as the backbone, 
        with random noise for data augmentation.

        Figure \ref{cd} shows the test accuracy versus the number of epochs.
        We can see that as the number of epochs increases, overfitting occurs in both Mean Teacher and the fully supervised learning for the source domain.
        From this figure, we can see that DNLL not only avoids the overfitting phenomenon but also is superior to Dual Student, and its performance is very close to that of the target domain supervision.

    \subsection{Execution time of DNLL} 

    \begin{table}[htbp]
        \centering
        \caption{The execution time (seconds) of DNLL and other competitive methods such as Mean Teacher (MT) and Dual Student (DS).}
        \label{tab:7}      
        \begin{tabular}{lllll}
        \hline\noalign{\smallskip}
          & MT & DS & DNLL\\
        \noalign{\smallskip}\hline\noalign{\smallskip}
        Train iteration time  & 0.072 & 0.145 & 0.143\\
        Inference iteration time  & 0.0183 & 0.0189 & 0.0184 \\
        \noalign{\smallskip}\hline
        \end{tabular}
    \end{table}
        
        In this section, we conduct experiments to investigate the execution time of DNLL.
        We report the average time for each iteration during training and testing.
        We evaluate the execution time with the CIFAR-10 dataset using 4000 randomly selected training samples as labeled data. The batch size is set to 100. 
        The number of both labeled and unlabeled data in a batch is 50.
        We compare DNLL with Mean Teacher and Dual Student in the same settings in terms of execution time.
        The experiment is performed on a GTX 3060 GPU with Pytorch-1.10.2 software toolkit. 
        The system memory is 64 GB, and the CPU is Intel Core i5-11400F. 
        The experimental results are shown in Table \ref{tab:7} and Figure \ref{time}.

        From Table \ref{tab:7} and Figure \ref{time}, we can see that Mean Teacher takes the shortest training time but produces the lowest testing accuracy on the testing set. As both DNLL and Dual Student use a dual model
        structure, the train time for each iteration is approximately twice that
        of Mean Teacher, but both have higher accuracy. The training
         time of DNLL and Dual Student are similar, but the
        performance of DNLL is higher than that of Dual Student.
        The average testing time of each iteration is shown in Table \ref{tab:7}. Due to the similarity in model architectures, the testing time of all methods is similar.

        \begin{figure}[htbp]
            \centering
            \includegraphics[height=6cm, width=8cm]{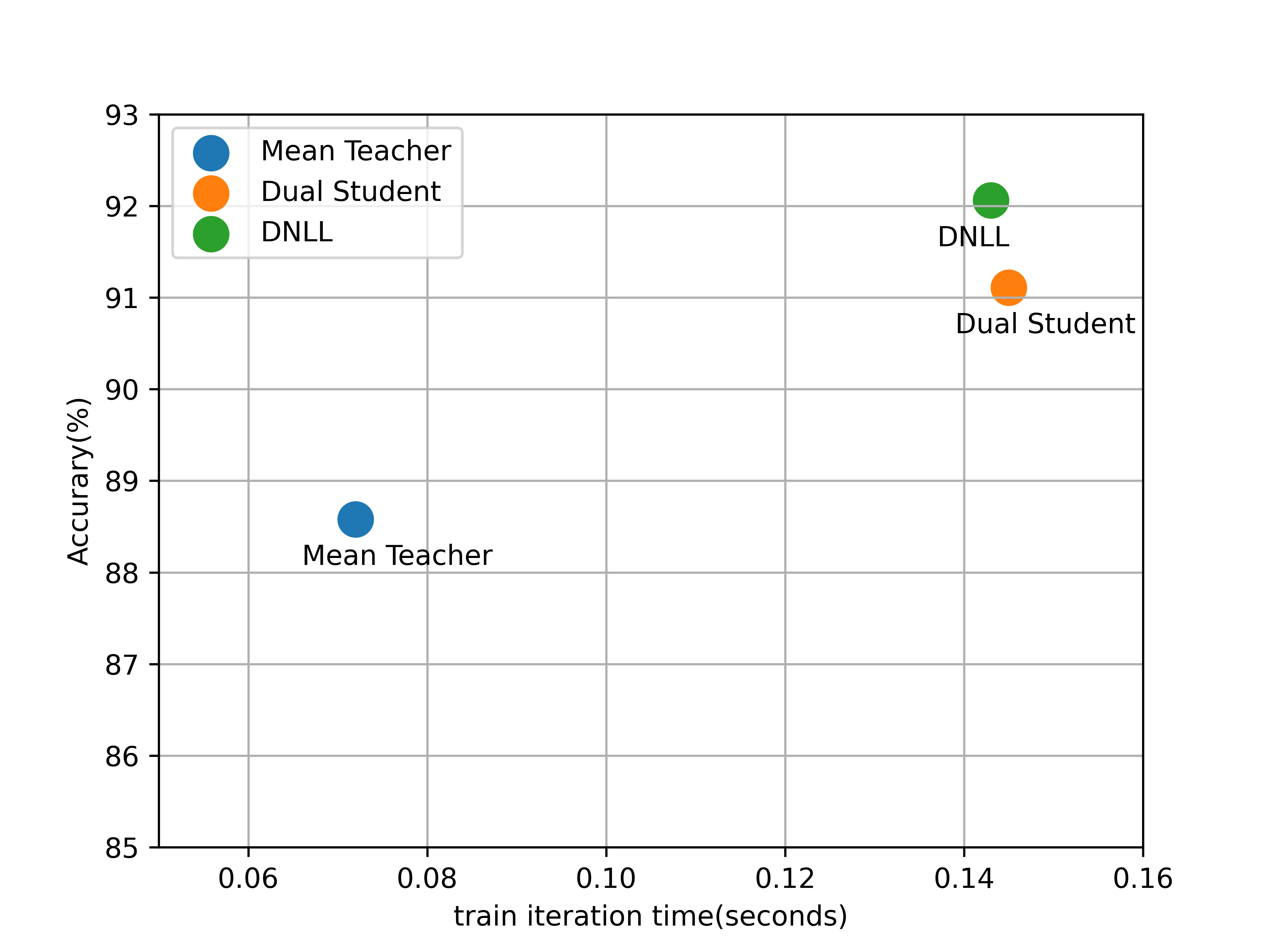}
            \caption{The training time (seconds) for each iteration and the testing accuracies of
            DNLL, Mean Teacher and Dual Student.}
            \label{time}
        \end{figure}

\section{Discussions}
    Our approach has several advantages over existing semi-supervised algorithms. 
    Firstly, in semi-supervised learning, our approach outperforms state-of-the-art approaches on benchmarks. 
    Secondly, the unsupervised learning part of our methods can easily be used as add-ons for other semi-supervised learning methods to improve their performance. 
    Finally, our approach fits domain adaptation tasks as well.
    We discuss the differences between DNLL and other methods that use a dual model.

    Mean Teacher (MT): MT \cite{tarvainen2017mean} has been proposed to improve the temporal-ensembling model \cite{laine2016temporal}. 
    The framework of MT  consists of a student model and a teacher model. 
    The student model is trained by perturbing the input data. The output of the student model is trained to be consistent with the output of the teacher model.
    Different from DNLL,  in MT, the teacher model is only updated by EMA.
    Thus, the predictions between the teacher model and the student model converge to be the same relatively fast during training. 
    In addition, submodels in DNLL can generate pseudo-negative labels to help each other filter out erroneous information, while the student model and the teacher model in MT cannot. 

    Dual Student (DS): DS \cite{ke2019dual} has been proposed to improve MT.
    DS trains two submodels online simultaneously with different initialization parameters in order to avoid coupling between the two models in the early training stages.
    To transfer reliable knowledge, submodels in DS filter unlabeled data with low prediction confidences or inter-submodel consistency. 
    This can lead to an underutilization of a significant amount of unlabeled data. 
    On the other hand, in DNLL, most of the unlabeled data can be used in the training process, and the transferring of erroneous information is also reduced by using pseudo-negative labels.

    Mutual Learning of Complementary Networks: This method proposes a complementary correction network (CCN) \cite{wu2019mutual} based on Deep Mutual Learning (DML) \cite{zhang2018deep}. 
    This method simultaneously trains three submodels, including two submodels with the same structure and one CCN. 
    The CCN takes the output from one submodel and the intermediate features extracted by another submodel as input and is trained with labeled data only. This network is then used to correct predictions by submodels.
    The prediction is then used as pseudo-labels for one of the submodels.
    The performance of the CCN can significantly determine the quality of the pseudo label, which in turn affects the training of the underlying submodel. 
    On the other hand,  DNLL is trained in a much simpler and more effective way. 

    Dynamic Mutual Training (DMT): DMT \cite{feng2022dmt} uses a weighted loss to control the selection of unlabeled data items so that data items with inconsistent predictions by submodels are filtered in the loss calculation.
    In addition, this method uses a course learning strategy in which unlabeled data are gradually used in the training process rather than used as a whole from the beginning.  
    Compared with DNLL, this method also suffers from the underutilization of unlabeled data, and it is also time-consuming to train repetitively during course learning. 
      
\section{Conclusion}

    The paper analyzes submodel degeneration and underutilization problems suffered from traditional mutual learning approaches. 
    To address these problems, we propose a novel mutual learning method for semi-supervised learning.
    Submodels in this approach provide each other with pseudo-negative labels instead of traditional pseudo labels. It can reduce error accumulation and promote unlabeled data utilization and is justified theoretically and experimentally.
    We also propose the error perception mechanism to help select efficient pseudo-negative labels. 
    This framework can also be useful in different tasks.

\section*{Acknowledgements}
This work was supported by the Natural Science Foundation of Zhejiang Province (NO. LGG20F020011), Ningbo Science and Technology Innovation Project (No. 2022Z075), and Open Fund by Ningbo Institute of Materials Technology \& Engineering, the Chinese Academy of Sciences.

% \bibliographystyle{elsarticle-num}
\bibliographystyle{cas-model2-names}  
\bibliography{cas-refs}

\end{document}